\def\year{2019}\relax
\pdfoutput=1
\documentclass[letterpaper]{article} %DO NOT CHANGE THIS
\usepackage{aaai19}  %Required
\usepackage{times}  %Required
\usepackage{helvet}  %Required
\usepackage{courier}  %Required
\usepackage{url}  %Required
\usepackage{graphicx}  %Required
\usepackage{microtype}
\usepackage{graphicx}
\usepackage{subfigure}
\usepackage{algorithm}
\usepackage[noend]{algpseudocode}

\usepackage{booktabs} % for professional tables
\usepackage{array}
\usepackage{todonotes}
\usepackage{amsmath, amsthm, amsfonts}
\usepackage{tikz}
\usepackage{color}

\newtheorem{definition}{Definition}[]

\newtheorem{corollary}{Corollary}[]
\newtheorem{prop}{Proposition}[]
 
\newcommand{\X}{\mathcal{X}}
\newcommand{\expeq}{\overset{\mathbb{E}}{=}}
\newcommand{\nexpeq}{\overset{\mathbb{E}}{\ne}}
\newcommand{\cramer}{\text{Cram\'er }}
\newcommand{\expect}{\mathbb{E}}
\usepackage{hyperref}
\DeclareMathOperator*{\argmax}{arg\,max}
\def \cA {\mathcal{A}}
\def \cD {\mathcal{D}}
\def \cX {\mathcal{X}}
\def \cZ {\mathcal{Z}}
\def \cQ {\mathcal{Q}}
\def \bN {\mathbb{N}}
\def \bR {\mathbb{R}}
\def \Var{\text{Var}}
\def \grad {\nabla}
\def \Rmax {R_{\textsc{max}}}
\begin{document}
\title{A Comparative Analysis of Expected and Distributional Reinforcement Learning}

 \author{Clare Lyle,\textsuperscript{1}
  Pablo Samuel Castro, \textsuperscript{2}
Marc G. Bellemare \textsuperscript{2}\\
\textsuperscript{1}{University of Oxford (work done while at Google Brain)}\\
\textsuperscript{2}{Google Brain}\\
clare.lyle@cs.ox.ac.uk,
 psc@google.com,
 bellemare@google.com}
\maketitle
\begin{abstract}

Since their introduction a year ago, distributional approaches to reinforcement learning (distributional RL) have produced strong results relative to the standard approach which models expected values (expected RL). However, aside from convergence guarantees, there have been few theoretical results investigating the reasons behind the improvements distributional RL provides. In this paper we begin the investigation into this fundamental question by analyzing the differences in the tabular, linear approximation, and non-linear approximation settings. We prove that in many realizations of the tabular and linear approximation settings, distributional RL behaves exactly the same as expected RL. In cases where the two methods behave differently, distributional RL can in fact hurt performance when it does not induce identical behaviour. We then continue with an empirical analysis comparing distributional and expected RL methods in control settings with non-linear approximators to tease apart where the improvements from distributional RL methods are coming from.
\end{abstract}

\section{Introduction}
The distributional perspective, in which one models the distribution of returns from a state instead of only its expected value, was recently introduced by \cite{bellemare2017}. The first distributional reinforcement learning algorithm, C51, saw dramatic improvements in performance in many Atari 2600 games when compared to an algorithm that only modelled expected values \cite{bellemare2017}. Since then, additional distributional algorithms have been proposed, such as quantile regression \cite{quantile_regression} and implicit quantile networks \cite{iqn}, with many of these improving on the results of C51. The abundance of empirical results make it hard to dispute that taking the distributional perspective is helpful in deep reinforcement learning problems, but theoretical motivation for this perspective is comparatively scarce.  Possible reasons for this include the following, proposed by \cite{bellemare2017} .
\begin{enumerate}
\item \textbf{Reduced chattering:} modeling a distribution may reduce prediction variance, which may help in policy iteration.
\item \textbf{Improved optimization behaviour:} distributions may present a more stable learning target, or in some cases (e.g. the softmax distribution used in the C51 algorithm) have a regularizing effect in optimization for neural networks.
\item \textbf{Auxiliary tasks:} the distribution offers a richer set of predictions for learning, serving as a set of auxiliary tasks which is tightly coupled to the reward.
\end{enumerate}
Initial efforts to provide a theoretical framework for the analysis of distributional algorithms demonstrated their convergence properties \cite{understanding-rl}, and did not directly compare their expected performance to expected algorithms. Indeed, even experimental results supporting the distributional perspective have largely been restricted to the deep reinforcement learning setting, and it is not clear whether the benefits of the distributional perspective also hold in simpler tasks. In this paper we continue lifting the veil on this mystery by investigating the behavioural differences between distributional and expected RL, and whether these behavioural differences necessarily result in an advantage for distributional methods.

\section{Background}
We model the reinforcement learning problem as an agent interacting with an environment so as to maximize cumulative discounted reward. We formalize the notion of an environment with a Markov Decision Process (MDP) defined as the tuple $(\mathcal{X}, \cA, R, P, \gamma)$, where $\mathcal{X}$ denotes the state space, $\cA$ the set of possible actions, $R:\cX\times\cA\rightarrow Dist([-R_{MAX}, R_{MAX}])$ is a stochastic reward function mapping state-action pairs to a distribution over a set of bounded rewards, $P$ the transition probability kernel, and $\gamma \in [0, 1)$ the discount factor.

We denote by $\pi:\cX\rightarrow Dist(\cA)$ a stochastic policy mapping states to a distribution over actions (i.e. $\pi(a|x)$ is the agent's probability of choosing action $a$ in state $x$). We will use the notation $Q$ to refer to state-action value function, which has the type $Q:\cX \times \cA \rightarrow \mathbb{R}$. The value of a specific policy $\pi$ is given by the value function $Q^\pi$, defined as the discounted sum of expected future rewards after choosing action $a$ from state $s$ and then following $\pi$
\begin{equation*}\label{qdef}
    Q^\pi(x,a) := \mathbb{E}_{\pi, P} \bigg [\sum_{t=0}^\infty \gamma^t R(x_t, a_t) \bigg | x_0 = x, a_0 = a \bigg ].
\end{equation*}
One can also express this value as the fixed point of the Bellman operator $T^\pi$ \cite{bellman}, defined as 
\begin{align*}\label{expbellman}
\small
T^\pi Q(x,a) := &\expect[R(x,a)]
\\ &+ \gamma \sum_{x', a'} P(x'|x,a)\pi(a'|x')Q(x', a').
\end{align*}
The Bellman operator $T^\pi$ depends on the policy $\pi$, and is used in \textit{policy evaluation} \cite{suttonandbarto}. When we seek to improve the current policy, we enter the \textit{control} setting. In this setting, we modify the previous Bellman operator to obtain the Bellman optimality operator $T^*$, given by
\begin{equation*}\label{expbellmanopt}
T^* Q(x,a) := \expect[ R(x,a) ] + \gamma \sum_{x'}P(x'|x,a)\max_{a'}Q(x', a').
\end{equation*}

In many problems of interest, we do not have a full model of the MDP, and instead use a family of stochastic versions of the Bellman operator called temporal difference (TD) updates \cite{suttonandbarto}. We will focus on the SARSA update \cite{rummery94online}, defined as follows. We fix a policy $\pi$ and let $(x_t, a_t, r_t, x_{t+1}, a_{t+1})$ be a sampled transition from the MDP, where $r_t \sim R(x_t, a_t)$ is a realized reward and $a_{t+1} \sim \pi(\cdot | x_{t+1})$. We let $\alpha_t$ be a step size parameter for time $t$. Then given a value function estimate $Q_t: \X \times \cA \rightarrow \mathbb{R}$ at time $t$, the SARSA update gives the new estimate $Q_{t+1}$:
\begin{equation}\label{td_update}\small
Q_{t+1}(x_t,a_t) = (1 - \alpha_t) Q_{t}(x_t,a_t) + \alpha_t(r_t + \gamma Q_{t}(x_{t+1}, a_{t+1})) .
\end{equation}

Under certain conditions on the MDP and $\alpha_t$, SARSA converges to $Q^\pi$ \cite{td-convergence}.

Semi-gradient SARSA updates extend the SARSA update from the tabular to the function approximation setting. We consider a parameter vector $\theta_t$, and feature vectors $\phi_{x,a}$ for each $(x,a)\in\cX\times\cA$ such that 
\[Q_t(x,a) = \theta_t^T \phi_{x,a}.\]
Given $\theta_t$, $\theta_{t+1}$ is given by the \textit{semigradient} update  \cite{suttonandbarto}
\begin{equation}\label{semigradient_q}
\theta_{t+1} := \theta_t - \alpha_t (\theta_t^T\phi_{x_t, a_t} - r_t + \gamma \theta_t^T\phi_{x_{t+1}, a_{t+1}})\phi_{x_t, a_t} .
\end{equation}

Instead of considering only the expected return from a state-action pair, one can consider the distribution of returns. We will use the notation $Z: \cX \times \cA \rightarrow \textit{Dist}(\mathbb{R})$ to denote a \textit{return distribution function}. We can then construct an analogous Bellman operator for these functions, as shown by \cite{bellemare2017} and termed the distributional Bellman operator, $T^\pi_\cD$:
\begin{equation}\label{rvbellman}
T^\pi_\cD Z(x,a) \overset{D}{=} R(x,a) + \gamma Z(X', A')
\end{equation}
where $X', A'$ are the random variables corresponding to the next state and action. This is equality in distribution, and not an equality of random variables. 

Analogous to the expected Bellman operator, repeated application of the distributional Bellman operator can be shown to converge to the true distribution of returns \cite{bellemare2017}. Later work showed the convergence of stochastic updates in the distributional setting \cite{understanding-rl}. The proof of convergence of the distributional Bellman operator uses a contraction argument, which in the distributional setting requires us to be careful about how we define the distance between two return distribution functions.

Probability divergences and metrics capture this notion of distance. The theoretical properties of some probability distribution metrics have been previously explored by \cite{gibbs02choosing}, and the Wasserstein distance in particular studied further in the context of MDPs \cite{ferns12methods} as well as in the generative model literature \cite{wassersteingan}. The Wasserstein distance also appears in the distributional reinforcement learning literature, but we omit its definition in favour of the related \cramer distance, whose properties make it more amenable to the tools we use in our analysis.

The C51 algorithm uses the cross-entropy loss function to achieve promising performance in Atari games; however, the role of the cross-entropy loss in distributional RL has not been the subject of much theoretical analysis. We will use primarily the \cramer distance \cite{szekely2003} in the results that follow, which has been studied in greater depth in the distributional RL literature. Motivations for the use of this distance have been previously outlined for generative models \cite{cramer}.
\begin{definition}[Cram\'er Distance]
Let $P, Q$ be two probability distributions with Cumulative Distribution Functions (CDFs) $F_P, F_Q$. The Cram\'er metric $\ell_2$ between $P$ and $ Q$ is defined as follows: \[\ell_2(P, Q) = \sqrt{\int_{\mathbb{R}} (F_P(x) - F_Q(x))^2dx}\]
We will overload notation and write equivalently $\ell_2(P, Q) \equiv \ell_2(F_P, F_Q)$ or, when $X$ and $Y$ are random variables with laws $P$ and $Q$, $\ell_2(P,Q)\equiv \ell_2(X, Y) $. 

\end{definition}

Practically, distributional reinforcement learning algorithms require that we approximate distributions. There are many ways one can do this, for example by predicting the quantiles of the return distribution  \cite{quantile_regression}. In our analysis we will focus on the class of categorical distributions with finite support. Given some fixed set $\mathbf{z} = \{z_1, \dots, z_K\} \in \mathbb{R}^K$ with $z_1 \leq z_2 \leq \dots \leq z_K$, a categorical distribution $P$ with support $\mathbf{z}$ is a mixture of Dirac measures on each of the $z_i$'s, having the form

\begin{equation} \label{eqn:mixture_diracs}
    P \in \cZ_z := \left\lbrace \sum_{i=1}^K \alpha_i \delta_{z_i} : \alpha_i \ge 0, \sum_{i=1}^K \alpha_i = 1 \right\rbrace .
\end{equation}

Under this class of distributions, the \cramer distance becomes a finite sum

\[\ell_2(F_P, F_Q) = \sqrt{\sum_{i=1}^{K-1} (z_{i+1} - z_i)(F_P(z_i) - F_Q(z_i))^2} \]

which amounts to a weighted Euclidean norm between the CDFs of the two distributions. When the atoms of the support are equally spaced apart, we get a scalar multiple of the Euclidean distance between the vectors of the CDFs.

We can use the \cramer distance to define a projection onto a fixed categorical support \textbf{z} \cite{understanding-rl}. 
\begin{definition}[Cram\'er Projection]
Let $\mathbf{z}$ be an ordered set of $K$ real numbers. For a Dirac measure $\delta_y$, the Cram\'er projection $\Pi_C(\delta_y)$ onto the support $\mathbf{z}$ is given by:
\[\Pi_C(\delta_y) = \begin{cases}
                \delta_{z_1} & \text{if $y\leq z_1$} \\
                 \frac{z_{i+1} - y}{z_{i+1} - z_i}\delta_{z_i} +\frac{y - z_i}{z_{i+1} - z_i}\delta_{z_{i+1}} 
                  & \text{if $z_i < y \leq z_{i+1}$} \\
  \delta_{z_K} & \text{if $y > z_{K}$}
  \end{cases}
  \]
\end{definition}
The operator $\Pi_C$ has two notable properties: first, as hinted by the name \textit{\cramer projection}, it produces the distribution supported on $\mathbf{z}$ which minimizes the \cramer distance to the original distribution. Second, if the support of the distribution is contained in the interval $[z_1, z_K]$, we can show that the \cramer projection preserves the distribution's expected value \ref{projectionexpectation}. It is thus a natural approximation tool for categorical distributions.

\begin{prop}\label{projectionexpectation}
Let $\mathbf{z} \in \mathbb{R}^k$, and $P$ be a mixture of Dirac distributions (see Eq. \ref{eqn:mixture_diracs}) whose support is contained in the interval $[z_1, z_K]$. Then the \cramer projection $\Pi_C (P)$ onto $\mathbf{z}$ is such that 
\[\mathbb{E}[\Pi_C(P)] = \mathbb{E}[P] \].
\end{prop}

The \cramer projection is implicit in the C51 algorithm, introduced by \cite{bellemare2017}. The C51 algorithm uses a deep neural network to compute a softmax distribution, then updates its weights according to the cross-entropy loss between its prediction and a sampled target distribution, which is then projected onto the support of the predicted distribution using the \cramer projection.

\section{The Coupled Updates Method}
We are interested in the behavioural differences (or lack thereof) between distributional and expected RL algorithms.
We will study these differences through a methodology where we couple the experience used by the update rules of the two algorithms.%\marc{Still not quite right, but will do for now.}

Under this methodology, we consider pairs of agents: one that learns a value function (the \emph{expectation learner}) and one that learns a value distribution (the \emph{distribution learner}). The output of the first learner is a sequence of value functions $Q_1, Q_2, \dots$. The output of the second is a sequence of value distributions $Z_1, Z_2, \dots$. More precisely, each sequence is constructed via some update rule:
\begin{equation*}
    Q_{t+1} := U_E(Q_t, \omega_t) \qquad  Z_{t+1} := U_D(Z_t, \omega_t),
\end{equation*}
from initial conditions $Q_0$ and $Z_0$, respectively, and where $\omega_t$ is drawn from some sample space $\Omega$.   These update rules may be deterministic (for example, an application of the Bellman operator) or random (a sample-based update such as TD-learning). Importantly, however, the two updates may be coupled through the common sample $\omega_t$. Intuitively, $\Omega$ can be thought of as the source of randomness in the MDP. 

Key to our analysis is to study update rules that are analogues of one another. If $U_E$ is the Bellman operator, for example, then $U_D$ is the distributional Bellman operator. More generally speaking, we will distinguish between \emph{model-based} update rules, which do not depend on $\omega_t$, and \emph{sample-based} update rules, which do. In the latter case, we will assume access to a scheme that generates sample transitions based on the sequence $\omega_1, \omega_2, \dots$, that is, a generator $G : \omega_1, \dots, \omega_t \mapsto (x_t, a_t, r_t, x_{t+1}, a_{t+1})$. Under this scheme, a pair $U_E, U_D$ of sampled-based update rules receive exactly the same sample transitions (for each possible realization); hence the \emph{coupled updates method}, inspired by the notion of coupling from the probability theory literature \cite{thorisson2000coupling}.

The main question we will answer is: \emph{which analogue update rules preserve expectations?} Specifically, we write

\begin{equation*}
    Z \expeq Q \iff \expect [ Z(x,a) ] = Q(x,a) \quad \forall \; (x, a) \in \cX \times \cA .
\end{equation*}
We will say that analogue rules $U_D$ and $U_E$ are \emph{expectation-equivalent} if, for all sequences $(\omega_t)$, and for all $Z_0$ and $Q_0$,
\begin{equation*}
    Z_0 \expeq Q_0 \iff Z_t \expeq Q_t \quad \forall t \in \bN.
\end{equation*}

Our coupling-inspired methodology will allow us to rule out a number of common hypotheses regarding the good performance of distributional reinforcement learning:

\noindent \textbf{Distributional RL reduces variance.} By our coupling argument, any expectation-equivalent rules $U_D$ and $U_E$ produce exactly the same sequence of expected values, \emph{along each sample trajectory}. The distributions of expectations relative to the random draws $\omega_1, \omega_2, \dots$ are identical and therefore, $\Var [ \expect Z_t(x, a) ] = \Var [ Q_t(x, a) ]$ everywhere, and $U_D$ does not produce lower variance estimates.

\noindent \textbf{Distributional RL helps with policy iteration.} One may imagine that distributional RL helps identify the best action. But if $Z_t \expeq Q_t$ everywhere, then also greedy policies based on $\argmax Q_t(x, \cdot)$ and $\argmax \expect Z_t(x, \cdot)$ agree. Hence our results (presented in the context of policy evaluation) extend to the setting in which actions are selected on the basis of their expectation (e.g. $\epsilon$-greedy, softmax).

\noindent \textbf{Distributional RL is more stable with function approximation.} We will use the coupling methodology to provide evidence that, at least combined with linear function approximation, distributional update rules do not improve performance.

\section{Analysis of Behavioural Differences}

Our coupling-inspired methodology provides us with a stable framework to perform a theoretical investigation of the behavioural differences between distributional and expected RL. We use it through a progression of settings that will gradually increase in complexity to shed light on what causes distributional algorithms to behave differently from expected algorithms. We consider three axes of complexity: 1) how we approximate the state space, 2) how we represent the distribution, and 3) how we perform updates on the predicted distribution function.

\subsection{Tabular models}
We first consider tabular representations, which uniquely represent the predicted return distribution at each state-action pair. We start with the simplest class of updates, that of the expected and distributional Bellman operators. Here and below we will write $\cZ$ for the space of bounded return distribution functions and $\cQ$ for the space of bounded value functions. We begin with results regarding two model-based update rules.
\begin{prop}\label{oplemma}
 Let $Z_0 \in \cZ$ and $Q_0 \in \cQ$, and suppose that $Z_0 \expeq Q_0$. If
 \begin{equation*}
 Z_{t+1}:= T^\pi_\cD Z_t \quad Q_{t+1} := T^\pi Q_t,
 \end{equation*}
then also $Z_t \expeq Q_t \; \forall \; t \in \mathbb{N}$.
\end{prop}
See the supplemental material for the proof of this result and those that follow.

We next consider a categorical, tabular representation where the distribution at each state-action pair is stored explicitly but, as per Equation (\ref{eqn:mixture_diracs}), restricted to the finite support $\mathbf{z} = \{z_1, \dots, z_K\}$, \cite{understanding-rl}. Unlike the tabular representation of Prop. \ref{oplemma}, this algorithm has a practical implementation; however, after each Bellman update we must project back the result into the space of those finite-support distributions, giving rise to a projected operator.

\begin{prop}\label{approxop}
Suppose that the finite support brackets the set of attainable value distributions, in the sense that $z_1 \le -\frac{\Rmax}{1 - \gamma}$ and $z_K \ge \frac{\Rmax}{1 - \gamma}$. Define the projected distributional operator
\begin{equation*}
    T_C^\pi := \Pi_C T^\pi_\cD .
\end{equation*}
Suppose $Z_0\expeq Q_0$, for $Z_0 \in \cZ_z, Q_0 \in \cQ$. If
\begin{equation*}
    Z_{t+1} := T^\pi_C Z_t \quad Q_{t+1} := T^\pi Q_t,
\end{equation*}
then also $Z_t \expeq Q_t \; \forall \; t \in \mathbb{N}$.
\end{prop}
Next, we consider sample-based update rules, still in the tabular setting. These roughly correspond to the Categorical SARSA algorithm whose convergence was established by \cite{understanding-rl}, with or without the projection step. Here we highlight the expectation-equivalence of this algorithm to the classic SARSA algorithm \cite{rummery94online}.

For these results we will need some additional notation. Consider a sample transition $(x_t, a_t, r_t, x_{t+1}, a_{t+1})$. Given a random variable $Y$, denote its probability distribution by $P_Y$ and its cumulative distribution function by $F_Y$, respectively. With some abuse of notation we extend this to value distributions and write $P_Z(x, a)$ and $F_Z(x,a)$ for the probability distribution and cumulative distribution function, respectively, corresponding to $Z(x,a)$. Finally, let $Z_t'(x_t,a_t)$ be a random variable distributed like the target $r_t + \gamma Z_t(x_{t+1}, a_{t+1})$, and write $\Pi_C Z_t'(x,a)$ for its Cram\'er projection onto the support $\mathbf{z}$.
\begin{prop}\label{mixture}
Suppose that $Z_0 \in \cZ, Q_0 \in \cQ$ and $Z_0 \expeq Q_0$. Given a sample transition $(x_t, a_t, r_t, x_{t+1}, a_{t+1})$ consider the mixture update
\begin{equation*}
    P_{Z_{t+1}}(x, a) := \left \{ \begin{array}{ll}
        (1 - \alpha_t) P_{Z_t}(x, a) + \alpha_t P_{Z_t'}(x_t, a_t) & \\
        P_{Z_t}(x, a) & \hspace{-7em} \text{if } x, a \ne x_t, a_t
    \end{array} \right .
\end{equation*}
and the SARSA update
\begin{equation*}
Q_{t+1}(x_t, a_t) := \left \{ \begin{array}{ll}
    Q_t(x_t, a_t) + \alpha_t \delta_t  &\\
    Q_t(x, a) & \text{if } x, a \ne x_t, a_t
\end{array} \right .
\end{equation*}
where $\delta_t := (r_t + \gamma Q_t(x_{t+1}, a_{t+1}) - Q_t(x_t, a_t))$, then also $Z_t \expeq Q_t \; \forall \; t \in \mathbb{N}$.
\end{prop}
\begin{prop}\label{mixture_projection}
Suppose that $Z_0 \in \cZ_z, Q_0 \in \cQ$, $Z_0 \expeq Q_0$, that $\mathbf{z}$ brackets the set of attainable value distributions, and $P_{Z_t'}$ in Prop. \ref{mixture} is replaced by the projected target $P_{\Pi_C Z_t'}$. Then also $Z_t \expeq Q_t \; \forall \; t \in \mathbb{N}$.
\end{prop}
Together, the propositions above show that there is no benefit, at least in terms of modelling expectations, to using distributional RL in a tabular setting when considering the distributional analogues of update rules in common usage in reinforcement learning.
%For conciseness, we will refer to the $Q$-value update of Prop. \ref{mixture} as \emph{the SARSA update} from here on.

Next we turn our attention to a slightly more complex case, where distributional updates correspond to a semi-gradient update. In the expected setting, the mixture update of Prop. \ref{mixture} corresponds to taking the semi-gradient of the squared loss of the temporal difference error $\delta_t$. While there is no simple notion of semi-gradient when $Z$ is allowed to represent arbitrary distributions in $\cZ$, there is when we consider a categorical representation, which is a finite object when $\cX$ and $\cA$ are finite (specifically, $Z$ can be represented by $|\cX||\cA|K$ real values).

To keep the exposition simple, in what follows we ignore the fact that semi-gradient updates may yield an object which is not a probability distribution proper. In particular, the arguments remain unchanged if we allow the learner to output a signed distribution, as argued by \cite{signed}.
\begin{definition}[Gradient of Cram\'er Distance]
Let $Z, Z'\in\cZ_z$ be two categorical distributions supported on $\mathbf{z}$. We define the gradient of the squared \cramer distance with respect to the CDF of $Z$, denoted $\nabla_F \ell_2^2(Z, Z') \in \mathbb{R}^K$ as follows:
\begin{equation*}
    \nabla_F \ell_2^2(Z, Z')[i] := \frac{\partial}{\partial F(z_i)} \ell_2^2(Z, Z').
\end{equation*}
Similarly,
\begin{equation*}
    \nabla_P \ell_2^2(Z, Z')[i] := \frac{\partial}{\partial P(z_i)} \ell_2^2(Z, Z').
\end{equation*}
\end{definition}
We say that the categorical support $\mathbf{z}$ is \emph{$c$-spaced} if $z_{i+1} - z_i = c$ for all $i$ (recall that $z_{i+1} \ge z_i$).
\begin{prop}\label{gradient_cdf}
Suppose that the categorical support $\mathbf{z}$ is $c$-spaced. Let $Z_0 \in \cZ, Q_0 \in \cQ$ be such that $Z_0\expeq Q_0$.
Suppose that $Q_{t+1}$ is updated according to the SARSA update with step-size $\alpha_t$. Let $Z_t'$ be given by $\Pi_C (r_t + \gamma Z_t(x_{t+1}, a_{t+1}))$. Consider the CDF gradient update rule
\begin{align*}
\small
    F_{Z_{t+1}}(x, a) := \left \{ \begin{array}{ll} \small
        F_{Z_t}(x, a) + \alpha_t' \nabla_F \ell_2^2(Z_t(x_t, a_t), Z_t'(x_t, a_t)) & \\
        F_{Z_t}(x, a) & \hspace{-8em} \text{if } x, a \ne x_t, a_t .
    \end{array} \right .
\end{align*}
If $\alpha_t' = \tfrac{\alpha_t}{2c}$, then also $Z_t \expeq Q_t \; \forall \; t \in \mathbb{N}$.
\end{prop}
Prop. \ref{gradient_cdf} shows that there exists a semi-gradient update which is expectation-equivalent to the SARSA update, with only a change of step-size required. While this is not too surprising (\cite{understanding-rl} remarked on the relationship between the mixture update of Prop. \ref{mixture} and the CDF semi-gradient), the result does highlight that the equivalence continues to hold even in gradient-based settings.
The resemblance stops here, however, and we now come to our first negative result.
\begin{prop}\label{mixture_pdf}
Suppose the CDF gradient in update rule of Prop. \ref{gradient_cdf} is replaced by the PDF gradient $\nabla_P \ell_2^2 (Z_t, Z_t')$. Then for each choice of step-size $\alpha'$ there exists an MDP $M$ and a time step $t \in \bN$ for which $Z_0 \expeq Q_0$ but $Z_t \nexpeq Q_t$.
\end{prop}
The counterexample used in the proof of Prop. \ref{mixture_pdf} illustrates what happens when the gradient is taken w.r.t. the probability mass: some atoms of the distribution may be assigned negative probabilities. Including a projection step does not rectify the issue, as the expectation of $Z_t$ remains different from $Q_t$.

\subsection{Linear Function Approximation}

In the linear approximation setting, we represent each state-action pair $x,a$ as a \textit{feature vector} $\phi_{x,a} \in \mathbb{R}^d$, %where typically $d < |X||A|$,
%determined by some mapping $\phi$. 
We wish to find a linear function given by a weight vector $\theta$ such that 
\begin{equation*}
    \theta^T \phi_{x,a} \approx Q^\pi(x,a) .
\end{equation*}
In the categorical distributional setting, $\theta$ becomes a matrix $W \in \bR^{K \times d}$. Here we will consider approximating the cumulative distribution function:
\begin{equation*}
    F_{Z(x,a)}(z_i) = W \phi_{x,a}[i].
\end{equation*}
We can extend this parametric form by viewing $F$ as describing the CDF of a mixture of Diracs (Equation \ref{eqn:mixture_diracs}). Thus, $F(z) = 0$ for $z < z_1$, and similarly $F(z) = F(z_k)$ for $z \ge z_k$; see \cite{signed} for a justification.
In what follows we further assume that the support $\mathbf{z}$ is $1$-spaced.

In this setting, there may be no $W$ for which $F_Z(x,a)$ describes a proper cumulative distribution function: e.g. $F(y)$ may be less than or greater than 1 for $y > z_k$. Yet, as shown by \cite{signed}, we can still analyze the behaviour of a distributional algorithm which is allowed to output improper distributions. In our analysis we will assume that all predicted distributions sum to 1, though they may assign negative mass to some outcomes.

We write $\cQ_\phi := \{ \theta^\top \phi : \theta \in \bR^d \}$ for the set of value functions that can be represented by a linear approximation over $\phi$. Similarly, $\cZ_\phi := \{ W \phi : W \in \bR^{K \times d} \}$ is the set of CDFs that are linearly representable. For $Z_t \in \cZ_\phi$, let $W_t$ be the corresponding weight matrix. As before, we define $Z_t'(x_t, a_t)$ to be the random variable corresponding to the projected target $\Pi_C [ r_t + \gamma Z_t(x_{t+1}, a_{t+1})]$.

\begin{prop}\label{linear}
Let $Z_0 \in \cZ_\phi$, $Q_0 \in \cQ_\phi$, and $\mathbf{z} \in \mathbb{R}^K$ such that \textbf{z} is 1-spaced. Suppose that $Z_0 \expeq Q_0$, and that $Z_0(x,a)[z_K] = 1$ $\forall x,a$. Let $W_t, \theta_t$ respectively denote the weights corresponding to $Z_t$ and $Q_t$. If $Z_{t+1}$ is computed from the semi-gradient update rule
\begin{equation*}
    W_{t+1} := W_t + \alpha_t (F_{Z_t'}(x_t, a_t) - W_t\phi_{x_t, a_t} ) \phi_{x_t, a_t}^T
\end{equation*}
and $Q_{t+1}$ is computed according to Equation \ref{semigradient_q} with the same step-size $\alpha_t$, then also $Z_t \expeq Q_t \; \forall \; t \in \mathbb{N}$.
\end{prop}
Importantly, the gradient in the previous proposition is taken with respect to the CDF of the distribution. Taking the gradient with respect to the Probability Mass Function (PMF) does not preserve the expected value of the predictions, which we have already shown in the tabular setting. This negative result is consistent with the results on signed distributions by \cite{signed}.

\begin{figure*}[!htb]
\centering
\begin{subfigure}
  \centering
  \includegraphics[width=.4\textwidth]{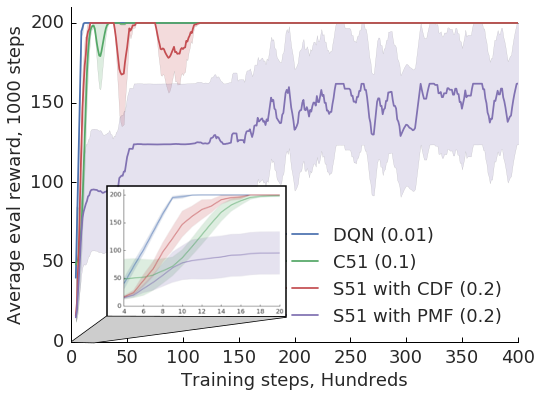}
\end{subfigure}
\begin{subfigure}
  \centering
  \includegraphics[width=.4\textwidth]{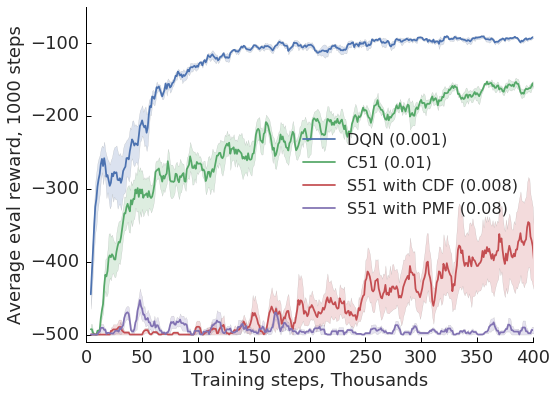}
\end{subfigure}
\caption{Cartpole (left) and Acrobot (right) with Fourier features of order 4. Step size in parentheses. Algorithms correspond to the lite versions described in text.}
\label{fig:fourier}
\end{figure*}

\begin{figure*}[!htb]
\centering
\begin{subfigure}
  \centering
  \includegraphics[width=.3\textwidth]{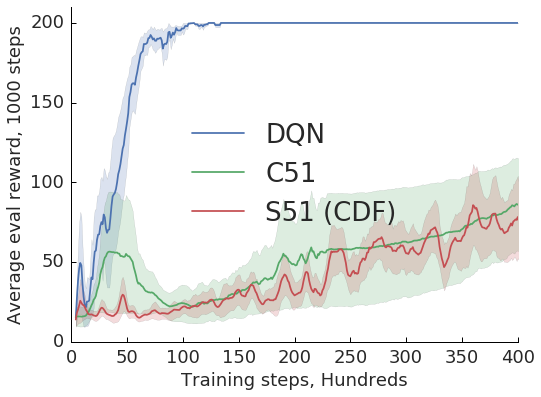}
\end{subfigure}
\begin{subfigure}
  \centering
  \includegraphics[width=.3\textwidth]{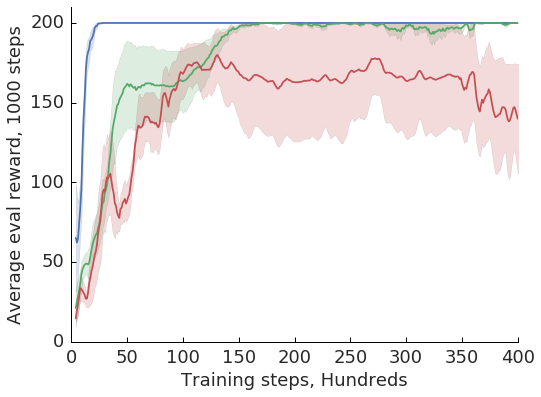}
\end{subfigure}
\begin{subfigure}
  \centering
  \includegraphics[width=.3\textwidth]{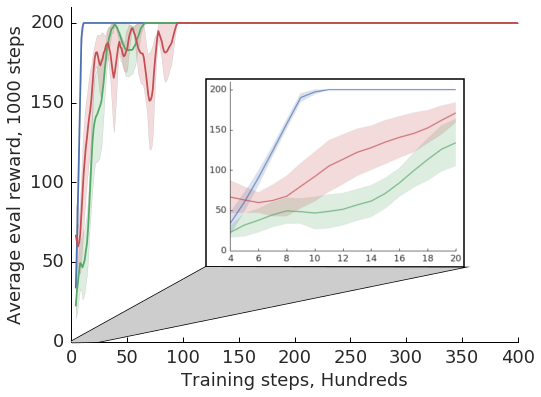}
\end{subfigure}
\caption{Varying the basis orders on CartPole. Orders 1 (left), 2 (center), 3 (right). Step sizes same as in Fig.\ref{fig:fourier}.}
\label{fig:varying_basis}
\end{figure*}

\subsection{Non-linear Function Approximation}
To conclude this theoretical analysis, we consider more general function approximation schemes, which we will refer to as the non-linear setting. In the non-linear setting, we consider a differentiable function $\psi(W, \cdot)$. For example, the function $\psi(W, \phi_{x,a})$ could be a probability mass function given by a softmax distribution over the logits $W\phi_{x,a}$, as is the case for the final layer of the neural network in the C51 algorithm.
\iffalse
\begin{definition}
Let $\psi_W : \mathbb{R}^d \rightarrow \mathbb{R}^K$ be parametrized by $W \in \mathbb{R}^m$. The nonlinear semigradient update with respect to $W$ is
That is, 
\begin{align*}
&\nabla_W \ell_2^2(\psi_W(x_t, a_t), F')[i] := \\&(F_t - \psi_W(x_t, a_t))(\frac{\partial}{\partial W[i]} \psi_W(x_t, a_t))^T.
\end{align*}
\end{definition}
\fi

\begin{prop}\label{nonlinear-lemma}
There exists a (nonlinear) representation of the cumulative distribution function parametrized by $W \in \bR^{K\times d}$ such that $Z_0 \expeq Q_0$ but after applying the semi-gradient update rule
\begin{equation*}
    W_{t+1} := W_t + \alpha_t  \grad_W \ell_2^2( \psi(W, \phi(x_t, a_t)), F_{Z_t'}),
\end{equation*}
where $F_{Z_t'}$ is the cumulative distribution function of the projected Bellman target,
we have $Z_1 \nexpeq Q_1$.
\end{prop}
The key difference with Prop. \ref{linear} is the change from the gradient of a linear function (the feature vector $\phi_{x_t, a_t}$) to the gradient of a nonlinear function; hence the result is not as trivial as it might look. Still, while the result is not altogether surprising, combined with our results in the linear case it shows that the interplay with nonlinear function approximation is a viable candidate for explaining the benefits of the distributional approach. In the next section we will present empirical evidence to this effect.

\section{Empirical Analysis}
Our theoretical results demonstrating that distributional RL often performs identically to expected RL contrast with the empirical results of \cite{bellemare2017,quantile_regression,iqn,barthmaron18distributional,hessel18rainbow}, to name a few. In this section we confirm our theoretical findings by providing empirical evidence that distributional reinforcement learning does not improve performance when combined with tabular representations or linear function approximation. Then, we find evidence of improved performance when combined with deep networks, suggesting that the answer lies, as suggested by Prop. \ref{nonlinear-lemma}, in distributional reinforcement learning's interaction with nonlinear function approximation.

\subsection{Tabular models}
Though theoretical results indicate that performing gradient updates with respect to the distribution's CDF should produce different predicted distributions from gradient updates with respect to its PMF, it is not immediately clear how these differences affect performance. To explore this, we considered a 12x12 gridworld environment and ran two distributional versions of Q-learning, one which performed gradient updates with respect to the CDF and one which performed updates with respect to the PMF of the predicted distribution, alongside traditional Q-learning. We found that, as predicted by Proposition 4, when given the same random seed the CDF update method had identical performance to traditional Q-learning. The PMF update method, though not significantly worse in performance, did exhibit \textit{different} performance, performing better on some random seeds and worse on others than expected Q-learning. We observed a similar phenomenon with a simple 3-state chain MDP. Results for both experiments are omitted for brevity, but are included in the supplemental material.

\begin{figure*}[!htb]
\centering
\begin{subfigure}
  \centering
  \includegraphics[width=.4\textwidth]{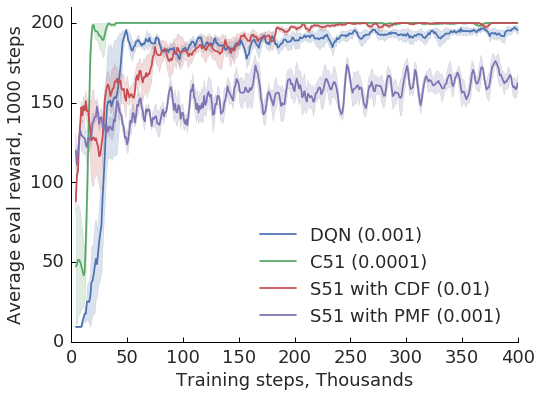}
\end{subfigure}
\begin{subfigure}
  \centering
  \includegraphics[width=.4\textwidth]{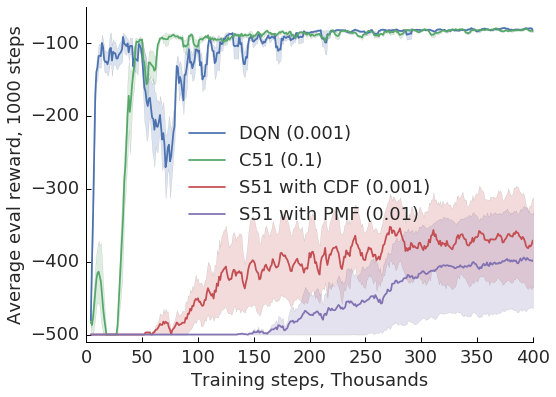}
\end{subfigure}
\caption{Cartpole with Adam optimizer (left) and Acrobot (right) with deep networks. Learning rate parameter in parentheses.}
\label{fig:deep}
\end{figure*}
%To better see how the PMF updates were changing the predicted distribution, we also considered a simple 3-state chain MDP with goal states to the left and right of the start state, for which the agent received a reward of 1 for taking the left (respectively right) action from that state and 0 otherwise, and examined the distribution predicted for the left action on the leftmost state. We observed that the PMF update method immediately introduced negative probabilities on atoms, and continued to do so for the 800 episodes in which we ran the algorithm. Visualizations are included in the appendix.

\subsection{Linear Function Approximation}
We next investigate whether the improved performance of distributional RL is due to the outputs and loss functions used, and whether this can be observed even with linear function approximation.

%To investigate this, we run a simple linear approximation experiment in the classic control environments CartPole and Acrobot, and use Fourier basis features %https://scholarworks.umass.edu/cgi/viewcontent.cgi?referer=https://www.google.ca/&httpsredir=1&article=1100&context=cs_faculty_pubs)
%of varying orders to encode the state for both environments.

We make use of three variants of established algorithms for our investigation, modified to use linear function approximators rather than deep networks. We include in this analysis an algorithm that computes a softmax over logits that are a linear function of a state feature vector. Although we did not include this type of approximator in the linear setting of our theoretical analysis, we include it here as it provides an analogue of C51 against which we can compare the other algorithms.
\begin{enumerate}
    \item \textbf{DQN-lite}, based on \cite{mnih2015humanlevel}, predicts $Q(x,a)$, the loss is the squared difference between the target and the prediction.
    \item \textbf{C51-lite}, based on \cite{bellemare2017}, outputs $Z(x,a)$, a softmax distribution whose logits are linear in $\phi_{x,a}$. It minimizes the cross-entropy loss between the target and the prediction.
    \item \textbf{S51-lite}, based on \cite{signed}, outputs $Z(x,a)$ as a categorical distribution whose probabilities are a linear function of $\phi_{x,a}$ and minimizes the squared \cramer distance.
\end{enumerate}

We further break S51-lite down into two sub-methods. One of these performs updates by taking the gradient of the \cramer distance with respect to the points of the PMF of the prediction, while the other takes the gradient with respect to the CDF. For a fair comparison, all algorithms used a stochastic gradient descent optimizer except where noted. We used the same hyperparameters for all algorithms, except for step sizes, where we chose the step size that gave the best performance for each algorithm. We otherwise use the usual agent infrastructure from DQN, including a replay memory of capacity 50,000 and a target network which is updated after every 10 training steps. We update the agent by sampling batches of 128 transitions from the replay memory.

We ran these algorithms on two classic control environments: CartPole and Acrobot. In CartPole, the objective is to keep a pole balanced on a cart by moving the cart left and right. In Acrobot, the objective is to swing a double-linked pendulum above a threshold by applying torque to one of its joints. We encode each original state $x$, ($x \in \mathbb{R}^4$ for CartPole and $x \in \mathbb{R}^6$ for Acrobot) as a feature vector $\phi(x)$ given by the Fourier basis for some fixed order \cite{konidaris11value}. For completeness, on Cartpole the basis of order 1 yields 15 features, order 2: 80 features, 3: 255, and finally 4: 624 features.

We first compare how DQN-, C51-, and S51-lite perform on the two tasks in Figure \ref{fig:fourier} with the order 4 basis, which is more than sufficient to well-approximate the value function. We observe that DQN learns more quickly than C51 and S51 with CDF, while S51 with PMF underperforms significantly. On Acrobot, the difference is even more pronounced. This result at first glance seems to contradict the theoretical result we showed indicating that S51-lite should perform similarly to DQN in the linear function approximation case, but we attribute this difference in performance to the fact that the initialization in the S51 algorithm doesn't enforce the assumption that the predicted distributions sum to 1.

We then investigate the effect of reducing the order in all algorithms in Figure \ref{fig:varying_basis}. We observe that the distributional algorithms performs poorly when there are too few features; by contrast, DQN-lite can perform both tasks with an order 2 basis. This indicates that the distributional methods suffered more when there were fewer informative features available than expectation-based methods did in this setting.

\subsection{Nonlinear function approximation}
We repeat the above experiment, but now replace the Fourier basis features with a two-layer ReLU neural network that is trained along with the final layer (which remains linear). In the CartPole task we found that DQN often diverged with the gradient descent optimizer, so we used Adam for all the algorithms, and chose the learning rate parameter that gave the best performance for each. Results are displayed in Figure \ref{fig:deep}. We can observe that C51 generally outperforms DQN, although they both eventually converge to the same value. It is interesting to notice that S51 has a harder time achieving the same performance, and comes nowhere near in the harder Acrobot task. This suggests that despite being theoretically unnecessary, the softmax in C51 is working to its advantage. This finding is consistent with the empirical results observed in Atari games by \cite{signed}.

The results of the previous two sets of experiments indicate that the benefits of distributional reinforcement learning are primarily gained from improved learning in the earlier layers of deep neural networks, as well as in the nonlinear softmax used in C51. We believe further investigations in this direction should lead to a deeper understanding of the distributional approach.

\section{Discussion and Future Work}

In this paper we have provided theoretical and empirical results that give evidence on the benefits (or, some cases, lack thereof) of the distributional approach in reinforcement learning. Together, our results point to function approximation as the key driver in the difference in behaviour between distributional and expected algorithms.

% Within the function approximation setting, we showed that there exists a class of distributional linear function approximators that behave identically to their expected counterparts, but that in the non-linear setting even this class fails to agree with expected methods. We conclude that one should expect to see the greatest difference between distributional and expected methods in deep reinforcement learning. 

%We then investigate empirically whether distributional methods always improve performance, or whether sometimes different behaviour can be worse. Experimental results agree with the intuitive reasoning that for simple function approximators, learning a distribution is more difficult than learning an expected value, and it should take longer for distributional methods to learn an optimal policy. With more complex function approximators, modeling a distribution in more stable behaviour and lead to better long-term performance. We also note that we have proved that, under some special conditions, distributional methods can result in sub-optimal performance, whereas the expected methods do not.

To summarize, our findings are:
\begin{enumerate}
    \item Distributional methods are generally expectation-equivalent when using tabular representations or linear function approximation, but
    \item diverge from expected methods when we use non-linear function approximation. 
    \item Empirically, we provide fresh confirmation that modelling a distribution helps when using non-linear approximation.
    \end{enumerate}

There are a few notions from distributional reinforcement learning not covered by our study here, including the effect of using Wasserstein projections of the kind implied by quantile regression \cite{quantile_regression}, and the impact of the softmax transfer function used by C51 on learning. 
In particular the regression setting, \cite{histogramloss} show that even for a fixed set of features, optimizing a distributional loss results in better accuracy than minimizing the squared error of predictions.

Yet we believe the most important question raised by our work is: what happens in deep neural networks that benefits most from the distributional perspective? Returning to the proposed reasons for the distributional perspective's success in the introduction, we note that the potentially regularizing effect of modeling a distribution, and a potential role as an `auxiliary task' played by the distribution are both avenues that remain largely unaddressed by this work.

\bibliographystyle{aaai}
\bibliography{arxiv}

\begin{thebibliography}{}

\bibitem[\protect\citeauthoryear{Arjovsky, Chintala, and
  Bottou}{2017}]{wassersteingan}
Arjovsky, M.; Chintala, S.; and Bottou, L.
\newblock 2017.
\newblock {W}asserstein generative adversarial networks.
\newblock In {\em Proceedings of the 34th International Conference on Machine
  Learning}, volume~70 of {\em Proceedings of Machine Learning Research},
  214--223.

\bibitem[\protect\citeauthoryear{Barth-Maron \bgroup et al\mbox.\egroup
  }{2018}]{barthmaron18distributional}
Barth-Maron, G.; Hoffman, M.~W.; Budden, D.; Dabney, W.; Horgan, D.; TB, D.;
  Muldal, A.; Heess, N.; and Lillicrap, T.
\newblock 2018.
\newblock Distributional policy gradients.
\newblock In {\em Proceedings of the International Conference on Learning
  Representations, to appear}.

\bibitem[\protect\citeauthoryear{Bellemare \bgroup et al\mbox.\egroup
  }{2017}]{cramer}
Bellemare, M.~G.; Danihelka, I.; Dabney, W.; Mohamed, S.; Lakshminarayanan, B.;
  Hoyer, S.; and Munos, R.
\newblock 2017.
\newblock The cramer distance as a solution to biased wasserstein gradients.
\newblock {\em CoRR} abs/1705.10743.

\bibitem[\protect\citeauthoryear{Bellemare \bgroup et al\mbox.\egroup
  }{2019}]{signed}
Bellemare, M.~G.; Roux, N.~L.; Castro, P.~S.; and Moitra., S.
\newblock 2019.
\newblock Distributional reinforcement learning with linear function
  approximation.
\newblock {\em To appear in Proceedings of AISTATS}.

\bibitem[\protect\citeauthoryear{Bellemare, Dabney, and
  Munos}{2017}]{bellemare2017}
Bellemare, M.~G.; Dabney, W.; and Munos, R.
\newblock 2017.
\newblock A distributional perspective on reinforcement learning.
\newblock In Precup, D., and Teh, Y.~W., eds., {\em Proceedings of the 34th
  International Conference on Machine Learning}, volume~70 of {\em Proceedings
  of Machine Learning Research},  449--458.
\newblock International Convention Centre, Sydney, Australia: PMLR.

\bibitem[\protect\citeauthoryear{Bellman}{1957}]{bellman}
Bellman.
\newblock 1957.
\newblock Dyamic programming.

\bibitem[\protect\citeauthoryear{Bertsekas}{1996}]{td-convergence}
Bertsekas, I.
\newblock 1996.
\newblock Temporal differences-based policy iteration and applications in
  neuro-dynamic programming.

\bibitem[\protect\citeauthoryear{Dabney \bgroup et al\mbox.\egroup
  }{2017}]{quantile_regression}
Dabney, W.; Rowland, M.; Bellemare, M.~G.; and Munos, R.
\newblock 2017.
\newblock Distributional reinforcement learning with quantile regression.
\newblock {\em CoRR} abs/1710.10044.

\bibitem[\protect\citeauthoryear{Dabney \bgroup et al\mbox.\egroup
  }{2018}]{iqn}
Dabney, W.; Ostrovski, G.; Silver, D.; and Munos, R.
\newblock 2018.
\newblock Implicit quantile networks for distributional reinforcement learning.
\newblock 80:1096--1105.

\bibitem[\protect\citeauthoryear{Ferns \bgroup et al\mbox.\egroup
  }{2012}]{ferns12methods}
Ferns, N.; Castro, P.~S.; Precup, D.; and Panangaden, P.
\newblock 2012.
\newblock Methods for computing state similarity in markov decision processes.
\newblock {\em CoRR} abs/1206.6836.

\bibitem[\protect\citeauthoryear{Gibbs and Su}{2002}]{gibbs02choosing}
Gibbs, A., and Su, F.
\newblock 2002.
\newblock On {C}hoosing and {B}ounding {P}robability {M}etrics.
\newblock {\em International {S}tatistical {R}eview} 30:419--435.

\bibitem[\protect\citeauthoryear{Hessel \bgroup et al\mbox.\egroup
  }{2018}]{hessel18rainbow}
Hessel, M.; Modayil, J.; van Hasselt, H.; Schaul, T.; Ostrovski, G.; Dabney,
  W.; Horgan, D.; Piot, B.; Azar, M.; and Silver, D.
\newblock 2018.
\newblock Rainbow: Combining improvements in deep reinforcement learning.
\newblock In {\em Proceedings of the AAAI Conference on Artificial
  Intelligence}.

\bibitem[\protect\citeauthoryear{Imani and White}{2018}]{histogramloss}
Imani, E., and White, M.
\newblock 2018.
\newblock Improving regression performance with distributional losses.
\newblock 80.

\bibitem[\protect\citeauthoryear{Konidaris, Osentoski, and
  Thomas}{2011}]{konidaris11value}
Konidaris, G.; Osentoski, S.; and Thomas, P.~S.
\newblock 2011.
\newblock Value function approximation in reinforcement learning using the
  fourier basis.
\newblock In {\em Proceedings of the AAAI Conference}.

\bibitem[\protect\citeauthoryear{Mnih \bgroup et al\mbox.\egroup
  }{2015}]{mnih2015humanlevel}
Mnih, V.; Kavukcuoglu, K.; Silver, D.; Rusu, A.~A.; Veness, J.; Bellemare,
  M.~G.; Graves, A.; Riedmiller, M.; Fidjeland, A.~K.; Ostrovski, G.; Petersen,
  S.; Beattie, C.; Sadik, A.; Antonoglou, I.; King, H.; Kumaran, D.; Wierstra,
  D.; Legg, S.; and Hassabis, D.
\newblock 2015.
\newblock Human-level control through deep reinforcement learning.
\newblock {\em Nature} 518(7540):529--533.

\bibitem[\protect\citeauthoryear{Rowland \bgroup et al\mbox.\egroup
  }{2018}]{understanding-rl}
Rowland, M.; Bellemare, M.; Dabney, W.; Munos, R.; and Teh, Y.~W.
\newblock 2018.
\newblock An analysis of categorical distributional reinforcement learning.
\newblock In Storkey, A., and Perez-Cruz, F., eds., {\em Proceedings of the
  Twenty-First International Conference on Artificial Intelligence and
  Statistics}, volume~84 of {\em Proceedings of Machine Learning Research},
  29--37.
\newblock Playa Blanca, Lanzarote, Canary Islands: PMLR.

\bibitem[\protect\citeauthoryear{Rummery and Niranjan}{1994}]{rummery94online}
Rummery, G.~A., and Niranjan, M.
\newblock 1994.
\newblock On-line {Q}-learning using connectionist systems.
\newblock Technical report, Cambridge University Engineering Department.

\bibitem[\protect\citeauthoryear{Sutton}{2018}]{suttonandbarto}
Sutton, B.
\newblock 2018.
\newblock Reinforcement learning, 2nd edition.

\bibitem[\protect\citeauthoryear{Sz{\'e}kely}{2003}]{szekely2003}
Sz{\'e}kely, G.
\newblock 2003.
\newblock E-statistics: The energy of statistical samples.
\newblock {\em Bowling Green State University, Department of Mathematics and
  Statistics Technical Report} 3(05):1--18.

\bibitem[\protect\citeauthoryear{Thorisson}{2000}]{thorisson2000coupling}
Thorisson, H.
\newblock 2000.
\newblock {\em Coupling, Stationarity, and Regeneration}.
\newblock Probability and Its Applications. Springer New York.

\end{thebibliography}
\clearpage
\section{Proofs of main results}
\setcounter{prop}{0}
\begin{prop}\label{projectionexpectation}
Let $\mathbf{z} \in \mathbb{R}^k$, and $P$ be a mixture of Dirac distributions (see Eq. \ref{eqn:mixture_diracs}) whose support is contained in the interval $[z_1, z_K]$. Then the \cramer projection $\Pi_C (P)$ onto $\mathbf{z}$ is such that 
\[\mathbb{E}[\Pi_C(P)] = \mathbb{E}[P] \].
\end{prop}
\begin{proof}
We first prove the statement for a single dirac $\delta_y$. We let $z_i, z_{i+1}$ be such that $z_i \leq y \leq z_{i+1}$.
\begin{align*}
    \mathbb{E}[\Pi_C \delta_y] &= \mathbb{E}[\frac{z_{i+1} - y}{z_{i+1} - z_i}\delta_{z_i} + \frac{y - z_i}{z_{i+1} - z_i}\delta_{z_{i+1}}]\\
    &= \frac{z_{i+1} - y}{z_{i+1} - z_i}\mathbb{E}[\delta_{z_i}] +  \frac{y - z_i}{z_{i+1} - z_i}\mathbb{E}[\delta_{z_{i+1}}]\\
    &= \frac{1}{z_{i+1} - z_i}(z_{i+1} - z_i)y \\
    &= y
\end{align*}
If the law of the random variable $Z$ is a mixture of diracs, i.e. $P = \sum_{i=1}^n \alpha_i\delta_{y_i}$, then we have 
\begin{align*}
    \mathbb{E}[\Pi_C Z] &= \mathbb{E}[\Pi_C \sum_{i=1}^n \alpha_i \delta_{y_i} ]\\
    &=\sum_{i=1}^n \alpha_i \mathbb{E}[\Pi_C \delta_{y_i}] \\
    &= \sum_{i=1}^n \alpha_i \mathbb{E}[ \delta_{y_i}]\\
    &= \mathbb{E}[Z]
\end{align*}
\end{proof}
\begin{prop}\label{oplemma}
 Let $Z_0 \in \cZ$ and $Q_0 \in \cQ$, and suppose that $Z_0 \expeq Q_0$. If
 \begin{equation*}
 Z_{t+1}:= T^\pi_\cD Z_t \quad Q_{t+1} := T^\pi Q_t,
 \end{equation*}
then also $Z_t \expeq Q_t \; \forall \; t \in \mathbb{N}$.
\end{prop}

\begin{proof}
By induction. By construction this is the case for $Z_0, Q_0$. Suppose it holds for timestep $t$. Then for timestep $t+1$, we have:
\begin{align*}
    \mathbb{E}&[ Z_{t+1}(x,a)] = \mathbb{E}[R(x,a) + Z_t(X', A')] \\
    &= \mathbb{E}[R(x,a)] +\gamma \sum_{x', a'}P(x'|x,a) \pi(a'|x') \mathbb{E}[Z_t(x',a')] \\
    &= \mathbb{E}[R(x,a)] + \gamma \sum_{x', a'} P(x'|x,a) \pi(a'|x')  Q_t(x',a') \\
    &= Q_{t+1}(x,a) \qedhere
\end{align*}

\end{proof}

\begin{prop}\label{approxop}
Suppose that the finite support brackets the set of attainable value distributions, in the sense that $z_1 \le -\frac{\Rmax}{1 - \gamma}$ and $z_K \ge \frac{\Rmax}{1 - \gamma}$. Define the projected distributional operator
\begin{equation*}
    T_C^\pi := \Pi_C T^\pi_\cD .
\end{equation*}
Suppose $Z_0\expeq Q_0$, for $Z_0 \in \cZ_z, Q_0 \in \cQ$. If
\begin{equation*}
    Z_{t+1} := T^\pi_C Z_t \quad Q_{t+1} := T^\pi Q_t,
\end{equation*}
then also $Z_t \expeq Q_t \; \forall \; t \in \mathbb{N}$.
\end{prop}
\begin{proof}
Again, we proceed by induction and observe that the equality is true by assumption for $t=0$. We use the result from Proposition \ref{projectionexpectation}. Then since $\mathbb{E}[T^\pi Z_t(x,a)] = T^\pi Q_t(x,a)$ we have that 
\begin{align*}
    \mathbb{E}[T^\pi_C Z_t(x,a)]& = \mathbb{E}[T^\pi Z_t(x,a)]\\
    &= T^\pi Q_t(x,a)
\end{align*} which proves the proposition.
\end{proof}

\begin{corollary}
The same proof can be used to show that the optimality operator $T^*$ induces equivalent behaviour in distributional and expected algorithms.
\end{corollary}
\begin{prop}\label{mixture}
Suppose that $Z_0 \in \cZ, Q_0 \in \cQ$ and $Z_0 \expeq Q_0$. Given a sample transition $(x_t, a_t, r_t, x_{t+1}, a_{t+1})$ consider the mixture update on the law of $Z_t$, denoted $P_{Z_t}$
\begin{equation*}
    P_{Z_{t+1}}(x, a) := \left \{ \begin{array}{ll}
        (1 - \alpha_t) P_{Z_t}(x, a) + \alpha_t P_{Z_t'}(x_t, a_t) & \\
        P_{Z_t}(x, a) & \hspace{-7em} \text{if } x, a \ne x_t, a_t
    \end{array} \right .
\end{equation*}
and the SARSA update
\begin{equation*}
Q_{t+1}(x_t, a_t) := \left \{ \begin{array}{ll}
    Q_t(x_t, a_t) + \alpha_t \delta_t  &\\
    Q_t(x, a) & \text{if } x, a \ne x_t, a_t
\end{array} \right .
\end{equation*}
where $\delta_t := (r_t + \gamma Q_t(x_{t+1}, a_{t+1}) - Q_t(x_t, a_t))$, then also $Z_t \expeq Q_t \; \forall \; t \in \mathbb{N}$.
\end{prop}
\begin{proof}
We proceed again by induction. We let $Z_t(x,a)$ be the return distribution at time $t$. By assumption, $\mathbb{E}[Z_0(x,a)] = Q_0(x,a)$ for all $x,a$. We suppose that each target and predicted distribution has finite support, and that the union of the supports for $Z_t$ and $Z_t'$ has size $k_t$. Now, for the induction step:
\begin{align*}
    \mathbb{E}&(Z_{t+1}(x_t,a_t)) = \sum_{i=1}^{k_t} P_{Z_{t+1}} (z_i) z_i \\
    &= \sum_{i=1}^{k_t} (1-\alpha_t) P_{Z_t}(z_i)z_i + \alpha_t P_{Z_t'}(z_i) z_i\\
    &= (1-\alpha_t) \sum_{i=1}^{k_t} P_{Z_t}(z_i)z_i + \alpha_t\sum_{i=1}^{k_t} P_{Z_t'}(z_i) z_i\\
    &= (1-\alpha_t)\mathbb{E}[Z_t(x_t, a_t)] + \alpha_t \mathbb{E}[r_{t+1} +\gamma Z_t(x_{t+1}, a_{t+1})]\\
    &= (1-\alpha_t) Q_t(x_t, a_t) + \alpha_t[r_t + \gamma Q_t(x_{t+1}, a_{t+1})] \\
    &= Q_{t+1}(x_t, a_t) \qedhere
\end{align*}
\end{proof}
\begin{prop}\label{mixture_projection}
Suppose that $Z_0 \in \cZ_z, Q_0 \in \cQ$, $Z_0 \expeq Q_0$, that $\mathbf{z}$ brackets the set of attainable value distributions, and $P_{Z_t'}$ in Prop. \ref{mixture} is replaced by the projected target $P_{\Pi_C Z_t'}$. Then also $Z_t \expeq Q_t \; \forall \; t \in \mathbb{N}$.
\end{prop}
\begin{proof}
Follows from propositions 1 and 4.
\end{proof}
\begin{prop}\label{gradient_cdf}
Suppose that the categorical support $\mathbf{z}$ is $c$-spaced. Let $Z_0 \in \cZ, Q_0 \in \cQ$ be such that $Z_0\expeq Q_0$.
Suppose that $Q_{t+1}$ is updated according to the SARSA update with step-size $\alpha_t$. Let $Z_t'$ be given by $\Pi_C (r_t + \gamma Z_t(x_{t+1}, a_{t+1}))$. Consider the CDF gradient update rule
\begin{align*}
\small
    F_{Z_{t+1}}(x, a) := \left \{ \begin{array}{ll} \small
        F_{Z_t}(x, a) + \alpha_t' \nabla_F \ell_2^2(Z_t(x_t, a_t), Z_t'(x_t, a_t)) & \\
        F_{Z_t}(x, a) & \hspace{-8em} \text{if } x, a \ne x_t, a_t .
    \end{array} \right .
\end{align*}
If $\alpha_t' = \tfrac{\alpha_t}{2c}$, then also $Z_t \expeq Q_t \; \forall \; t \in \mathbb{N}$.
\end{prop}
\begin{proof}
We first note that $\nabla_{F} \ell_2^2(F, F') = 2c(F' - F)$. 
\begin{align*}
    \nabla_F \ell_2^2(F, F')[i] &= \frac{\partial}{\partial F_i} \sum_{j=1}^{K-1}c(F'(z_j) - F(z_j))^2 \\
    &= \frac{\partial}{\partial F_i} c (F'(z_i) - F(z_i))^2 \\
    &= 2c(F'(z_i) - F(z_i))
\end{align*}
Thus the gradient update in this case is simply a mixture update with a different step size, and the result follows from Proposition 3.
\end{proof}
\begin{prop}\label{mixture_pdf}
Suppose the CDF gradient in update rule of Prop. \ref{gradient_cdf} is replaced by the PDF gradient $\nabla_P (Z_t, Z_t')$. Then for each choice of step-size $\alpha'$ there exists an MDP $M$ and a time step $t \in \bN$ for which $Z_0 \expeq Q_0$ but $Z_t \nexpeq Q_t$.
\end{prop}
\begin{proof}Suppose we have a support $\textbf{z}   = (0,1,2)$ and two CDFs: \[F' = (\frac{1}{2}, \frac{1}{2}, 1)\] and \[F = (\frac{1}{3}, \frac{2}{3}, 1).\] 
We note that the expected values of $F$ and $F'$ are both 1. Taking the gradient of the squared \cramer distance between the two distributions with respect to the PMF of the first gives $\nabla_p \ell_2^2(F, F') = (0, -\frac{1}{3}, 0)$. Now, when we consider 
\[P + \alpha \nabla \ell_2^2(P, P') \approx (\frac{1}{3},\frac{1}{3} -\frac{\alpha}{3}, \frac{1}{3})\]
 we can immediately observe that this is not a probability distribution as it has expected value $1 - \frac{\alpha}{3}$. The expectations of $P$ and $P'$ are both 1, so a \cramer distance update w.r.t. the CDFs would give a new distribution with expectation 1, as would an update which only looked at their expectations.
 \end{proof}
\begin{prop}\label{linear}
Let $Z_0 \in \cZ_\phi$ and $Q_0 \in \cQ_\phi$, and suppose that $Z_0 \expeq Q_0$. Let $W_t, \theta_t$ respectively denote the weights corresponding to $Z_t$ and $Q_t$. If $Z_{t+1}$ is computed from the semi-gradient update rule
\begin{equation*}
    W_{t+1} := W_t + \alpha (F_{Z_t'} - W_t\phi_{x_t, a_t} ) \phi_{x_t, a_t}^T
\end{equation*}
and $Q_{t+1}$ is computed according to Equation \ref{semigradient_q} with the same step-size $\alpha$, then also $Z_t \expeq Q_t \; \forall \; t \in \mathbb{N}$.
\end{prop}

\begin{proof}

We first note that we can compute the expected value of the distribution $W\phi_{x_t, a_t}$ directly by using the linear map $z^TC^{-1}$, where $C$ is a lower-triangular all-ones matrix (see \cite{signed} for details). So \[\mathbb{E}[Z_t(x_t, a_t)] = z^TC^{-1}W\phi_{x_t, a_t}.\]
We let $F_t$ and $v_t$ denote the distributional and expected TD targets respectively. Now, we observe that for any state-action pair $(x', a')$:
\begin{align*}
    \mathbb{E}[Z_{t+1}(x', a')] &= z^TC^{-1}W_{t+1}\phi_{x', a'} \\
    &= z^TC^{-1}(W_t \phi_{x', a'}) \\ &+\alpha z^TC^{-1}(F_t - W_t\phi_{x_t, a_t})\phi_{x_t, a_t}^T\phi_{x', a'}
\end{align*}
We note that by our assumption that all predicted distributions sum to 1, the expected value of the target signed measure $F_t$ given by $z^TC^{-1}F_t$ is equal to the target value function $v_t$. So to prove equality of expectation, it suffices to show
\begin{align*}&\alpha z^TC^{-1}(F_t - W_t\phi_{x_t, a_t})\phi_{x_t, a_t}^T \phi_{x', a'} =\\ &\alpha (\theta^T\phi_{x_t, a_t} - v_t)\phi_{x_t, a_t}^T\phi_{x', a'}.\end{align*}
We proceed as follows.
\begin{align*}
   &z^TC^{-1}( (W\phi_{x_t, a_t} - F_t)\phi_{x_t, a_t}^T)\phi_{x', a'} \\
    &=(z^TC^{-1}W\phi_{x_t, a_t} - z^TC^{-1}F_t)\phi_{x_t, a_t}^T\phi_{x', a'} \\
    \intertext{By assumption $Z_t \expeq Q_t$, and so $z^TC^{-1}W\phi_{x_t, a_t} = \theta^T\phi_{x_t, a_t}$. Further, as we also assume $z^TC^{-1}F_t = v_t$, we get that the difference becomes identical to the Q-value update.}
    &= (\theta^T\phi_{x_t, a_t} - v_t)\phi_{x_t, a_t}^T\phi_{x', a'} \qedhere \\
    % &= (\phi_{x_t, a_t}^T\phi_{x', a'}) (\theta^T \phi_{x_t, a_t} - v_t)\\
    % &= (\phi_{x_t, a_t}^T(\theta^T \phi_{x_t, a_t} - v_t)) \phi_{x', a'} \\
    % &= \nabla_\theta (v_\theta(x) - v)^2 \phi_{x', a'} \qedhere
\end{align*}
\end{proof}

\begin{prop}\label{nonlinear-lemma}
There exists a (nonlinear) representation of the cumulative distribution function parametrized by $W \in \bR^{K\times d}$ such that $Z_0 \expeq Q_0$ but after applying the semi-gradient update rule
\begin{equation*}
    W_{t+1} := W_t + \alpha  \grad_W \ell_2^2( \psi(W, \phi(x_t, a_t)), F_{Z_t'}),
\end{equation*}
where $F_{Z_t'}$ is the cumulative distribution function of the projected Bellman target,
we have $Z_1 \nexpeq Q_1$.
\end{prop} n
We present a concrete example where this is the case. For simplicity, the example will be one in which the target distribution is equal in expectation to the predicted distribution, but has a different law. Thus the update to the expected parameters will be zero, but the update to the distributional parameters will be non-zero, and if this update changes the expected value of the predicted distribution, then the new distributional prediction will disagree with the new expected prediction.

We will denote by $\sigma(y)$ the sigmoid function 
\[ \sigma(y) = \frac{1}{1 + e^{-y}}.\]
Let $\textbf{z} = (-1, 0, 1).$
Let $W = (w_1, w_2)$, and set \[\psi_W(x) := [\sigma(w_1 x_1), \sigma(w_2 x_2), 1]\] corresponding to $F(-1), F(0), F(1)$, with \[W_0 = [-\ln(2), -\ln(1/2)/2].\]

Set \[\psi_\theta(\phi_{x,a}) := z^TC^{-1} \psi_W(\phi_{x,a})\]
We sample a transition starting from $\phi_{x_t, a_t}$ and compute target distribution $F_t$ with values
\[\phi_{x_t, a_t} = (1,2) \text{ and }F = [0,1, 1]\] 
respectively. Then $\theta$ remains the same but the expected value of $\psi$ changes when we perform a gradient update. 
We first calculate the TD(0) semi-gradient with respect to the parameters $W$:
{\small
\begin{align*}
   \nabla_W[1] &= (F(-1) - F_\psi(-1))\frac{\partial}{\partial W_1}(F_\psi(-1))  \\
    &= (0 - \frac{1}{3}) \sigma(\ln(2))(1 - \sigma(\ln(2))) (-1) \\
    &=\frac{2}{27} \\
    \nabla_W[2] &=(F(0) - F_\psi(0))\frac{\partial}{\partial W_1}(F_\psi(0)) \\
    &= (1 - \frac{2}{3}) \sigma(-\ln(2))(1 - \sigma(-\ln(2))) (-2) \\
    &=\frac{-4}{27}
\end{align*}}
Let $\alpha = 1, W_{t+1} = W_t + \alpha \nabla_W \ell_2^2(\psi_W(\phi_{x,a}), F)$. Then we claim that the expected value of the new random variable $Z_{t+1}(1,2)$ denoted $F_{\psi '(1,2) }$ is different from the expectation of $F_{\psi(1,2)}$. To see this, consider:
\begin{align*}
    \mathbb{E}[Z_{t+1}(1,2)] &= (-1 )p_{\psi'(1,2)}(z_1) + (1 )p_{\psi'(1,2)}(z_3) \\
    &= -F_{\psi'}(z_1) + (1 - F_{\psi'}(z_2))\\
    &= \frac{-1}{1 + e^{(\ln(2) + 2/27)(1)}} + 1 \\
    &- \frac{1}{1 + e^{(\ln(\frac{1}{2})/2 -4/27)(2)}}\\ &\approx -0.05 \neq 0 = \mathbb{E}[Q_{t+1}(1,2)]
\end{align*}
And so $\mathbb{E}[Z_{t+1}(\phi_{x_t, a_t})] \neq Q_{t+1}(\phi_{x_t, a_t})$

\section{Additional Experimental Results}
We first present preliminary results from the gridworld experiment described in section 5.1. We set all of our agents to initially predict the same value for each state-action pair on each random seed, and then allow the agents to update their predictions using their respective update rules and take actions according to an $\epsilon$-greedy policy. We note that we no longer constrain all of the agents to take the same trajectory but couple them to all use the same random seed. Thus, if two agents always agree on the optimal action, they will attain the exact same performance in the gridworld. This is what we see in the plot below. Indeed, in the gridworld problem the difference between updating by the gradient of the PMF only marginally alters performance. The agent's objective in the gridworld environment is to reach the goal state in as few steps as possible, and fewer steps per episode indicates that the agent has learned the most direct route in the graph below. 
\begin{figure}
\centering
\includegraphics[width=200pt]{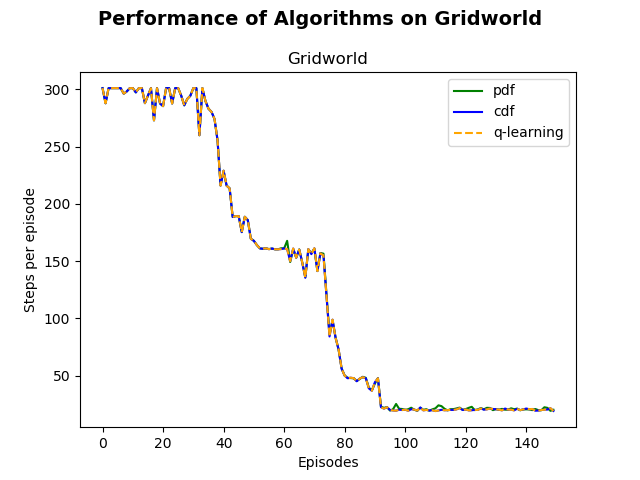}
 \caption{An environment where PMF updates perform well}
    \label{fig:gridworld}
\end{figure}

We see a larger disparity between CDF and PDF gradient updates in a 3-state MDP, where notably rewards are significantly less sparse. The agent's goal in this environment is to take the left action in the leftmost state or the right action in the rightmost state. In the 3-state MDP we relax the randomization coupling slightly and average over 5 runs in the MDP. We observe that although initially the PDF gradient updates perform similarly to the CDF and Q-learning gradient updates, they more often result in sub-optimal trajectories as training progresses. In contrast, the CDF updates continue to produce the same average performance as q-learning.
\begin{figure}[H]
    \centering
    \includegraphics[width=200pt]{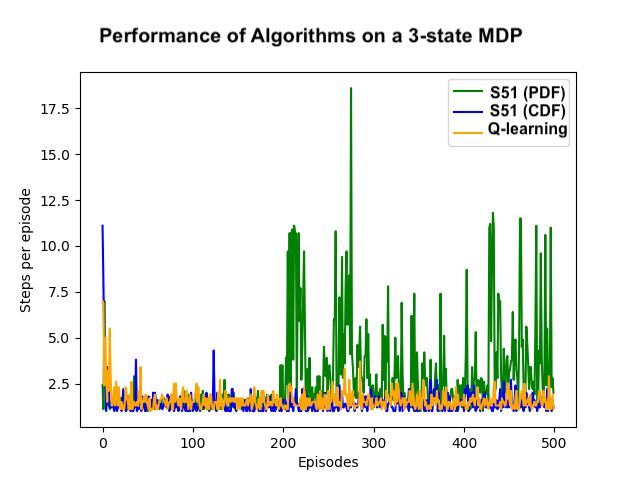}
    \caption{An environment where PMF updates suffer}
    \label{fig:threestate}
\end{figure}
\begin{figure}[H]
    \centering
    \includegraphics[width=100pt]{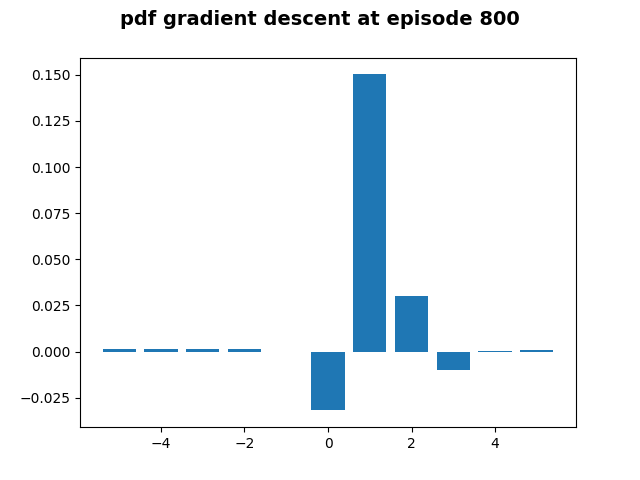}
    \includegraphics[width=100pt]{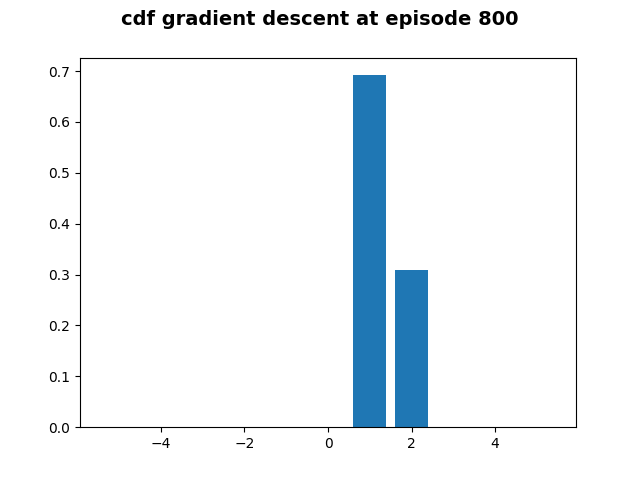}
    \caption{Predictions of value of goal state in 3-state MDP environment}
    \label{fig:predictions}
\end{figure}
We observe here that this worse performance occurs in conjunction with a predicted `distribution' that does not resemble a probability distribution, having negative probabilities which do not integrate to 1. In contrast, the predictions output by the agent which models the CDF are always proper distributions.

% \end{document}
\end{document}